\documentclass{amsart}
\usepackage[utf8]{inputenc}
\usepackage{sty}
\usepackage{enumitem}
\usepackage{placeins}
\usepackage{subcaption}
\usepackage{graphicx}
\usepackage{caption}
\usepackage{url}

\DeclareMathOperator{\atanh}{atanh}

\renewcommand{\defn}[1]{{\bf #1}}

\begin{document}

\title[Notes on computational-to-statistical gaps]{Notes on computational-to-statistical gaps: predictions using statistical physics}

\dedicatory{In memory of Amelia Perry, and her love for learning and teaching.}

\author{Afonso S.\ Bandeira}
\address[Bandeira]{Department of Mathematics and Center for Data Science, Courant Institute of Mathematical Sciences, New York University}
\email{bandeira@cims.nyu.edu}

\author{Amelia Perry}
\address[Perry]{Department of Mathematics, Massachusetts Institute of Technology}

\author{Alexander S.\ Wein}
\address[Wein]{Department of Mathematics, Massachusetts Institute of Technology}
\email{awein@mit.edu}

\thanks{ASB was partially supported by NSF DMS-1712730 and NSF DMS-1719545.}
\thanks{AP was supported in part by NSF CAREER Award CCF-1453261 and a grant from the MIT NEC Corporation.}
\thanks{ASW received Government support under and awarded by DoD, Air Force Office of Scientific Research, National Defense Science and Engineering Graduate (NDSEG) Fellowship, 32 CFR 168a.}

\begin{abstract}

In these notes we describe heuristics to predict computational-to-statistical gaps in certain statistical problems. These are regimes in which the underlying statistical problem is information-theoretically possible although no efficient algorithm exists, rendering the problem essentially unsolvable for large instances. The methods we describe here are based on mature, albeit non-rigorous, tools from statistical physics.

These notes are based on a lecture series given by the authors at the Courant Institute of Mathematical Sciences in New York City, on May $16^\text{th}$, 2017.

\end{abstract}

\maketitle

\section{Introduction}

Statistics has long studied how to recover information from data. Theoretical statistics is concerned with, in part, understanding under which circumstances such recovery is possible. Oftentimes recovery procedures amount to computational tasks to be performed on the data that may be computationally expensive, and so prohibitive for large datasets. While computer science, and in particular complexity theory, has focused on studying hardness of computational problems on worst-case instances, time and time again it is observed that computational tasks on data can often be solved far faster than the worst case complexity would suggest. This is not shocking; it is simply a manifestation of the fact that instances arising from real-world data are not adversarial. This illustrates, however, an important gap in fundamental knowledge: the understanding of \textbf{``computational hardness of statistical estimation problems''}.

For concreteness we will focus on the case where we want to learn a set of parameters from samples of a distribution, or estimate a signal from noisy measurements (often two interpretations of the same problem). In the problems we will consider, there is a natural notion of signal-to-noise ratio (SNR) which can be related to the variance of the distribution of samples, the strength of the noise, the number of samples or measurements obtained, the size of a hidden planted structure buried in noise, etc. Two ``phase transitions'' are often studied. Theoretical statistics and information theory often study the critical SNR below which it is statistically impossible to estimate the parameters (or recover the signal, or find the hidden structure), and we call this threshold $\mathrm{SNR_{Stat}}$. On the other hand, many algorithm development fields propose and analyze efficient algorithms to understand for which SNR levels different algorithms work. Despite significant effort to develop ever better algorithms, there are various problems for which no efficient algorithm is known to achieve recovery close to the statistical threshold $\mathrm{SNR_{Stat}}$. Thus we are interested in the critical threshold $\mathrm{SNR_{Comp}} \geq \mathrm{SNR_{Stat}}$ below which it is fundamentally impossible for an efficient (polynomial time) algorithm to recover the information of interest.

\vspace{-10pt}

\begin{figure}[!ht]
    \centering
        \includegraphics[width=1.1\linewidth]{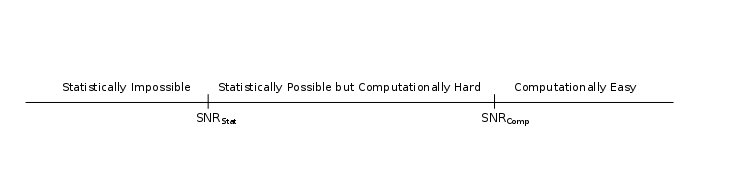}
    \label{fig:snr}
\end{figure}

\vspace{-10pt}

There are many problems believed to exhibit computational-to-statistical gaps. Examples include community detection~\cite{holland1983stochastic,decelle,abbe-detection,bmnn}, planted clique~\cite{Alon_Krivelevich_Sudakov_98,amp-clique,Baraketal_PlantedClique}, sparse principal component analysis~\cite{Berthet_Rigollet_AoS,Berthet_Rigollet_SparsePCAcolt,lkz-sparse}, structured spiked matrix models~\cite{lkz-mmse,pwbm-contig,mi,replica-proof,lm}, spiked tensor models~\cite{mr-tensor,Hopkins_Shi_Steurer_fastTensorPCASOS,pwb-tensor,lesieur2017statistical,KimBandeiraGoemans_sampta17}, and synchronization problems over groups~\cite{ASinger_2011_angsync,pwbm-contig,pwbm-amp}.

In these notes we will be concerned with predicting the locations of the thresholds $\mathrm{SNR_{Stat}}$ and $\mathrm{SNR_{Comp}}$ for Bayesian inference problems. In particular, we will focus on a couple of heuristics borrowed from statistical physics and illustrate them on two example problems: the Rademacher spiked Wigner problem (Example~\ref{ex:radwig}) and the related problem of community detection in the stochastic block model (Example~\ref{ex:cdsbm}). While we focus on these problems, we will try to cover the techniques in a way that conveys how they are broadly applicable.

At first glance, it may seem surprising that statistical physics has anything to do with Bayesian inference problems. The connection lies in the \emph{Gibbs (or Boltzmann) distribution} that is widely used in statistical physics to model disordered systems such as magnets. It turns out that in many Bayesian inference problems, the posterior distribution of the unknown signal given the data also follows a Gibbs distribution, and thus many techniques from statistical physics can be applied. More specifically, many inference problems follow similar equations to \emph{spin glasses}, which are physical systems in which the interaction strength between each pair of particles is random. The techniques that we borrow from statistical physics are largely non-rigorous but yield extremely precise predictions of both the statistical and computational limits. Furthermore, the predictions made by these heuristics have now been rigorously verified for many problems, and thus we have good reason to trust them on new problems. See the survey \cite{statmech-survey} for more on the deep interplay between statistical physics and inference.

Many techniques have been developed in order to understand computational-to-statistical gaps. We now give a brief overview of some of these methods, both the ones we will cover in these notes and some that we will not.

\subsubsection*{Reductions}
A natural approach to arguing that a task is computationally hard is via reductions, by showing that a problem is computationally hard conditioned on another problem being hard. This technique is extremely effective when studying the worst-case hardness of computational problems (a famous example being the list of 21 NP-hard combinatorial problems of Karp~\cite{Karp21problems}). There are also some remarkable successes in using this idea in the context of average-case problems (i.e.\ statistical inference on random models), starting with the work of Berthet and Rigollet on sparse PCA~\cite{Berthet_Rigollet_AoS,Berthet_Rigollet_SparsePCAcolt} and including also some conditional lower bounds for community detection with sublinear sized communities~\cite{Ma_Wu_complowerbounds,Hajek_Wu_Xu_complowerbounds}. These works show conditional hardness by reduction to the planted clique problem, which is widely believed to be hard in certain regimes. Unfortunately this method of reductions has so far been limited to problems that are fairly similar to planted clique.

\subsubsection*{Sum-of-squares hierarchy}
Sum-of-squares~\cite{Lassere_01_SOS,Parrilo_thesis_SOS,Nesterov_00_SOS,Shor_87_SOS,Barak_Steurer_surveyICM} is a hierarchy of algorithms to approximate solutions of combinatorial problems, or more generally, polynomial optimization problems. For each positive integer $d$, the algorithm at level $d$ of the hierarchy is a semidefinite program that relaxes the notion of a distribution over the solution space by only keeping track of moments of order $\le d$. As you go up the hierarchy (increasing $d$), the algorithms get more powerful but also run slower: the runtime is $n^{O(d)}$. The level-2 relaxation coincides with the algorithms in the seminal work of Goemans and Williamson~\cite{MXGoemans_DPWilliamson_1995} and Lovasz~\cite{Lovasz_ShannonCapacity}. The celebrated unique games conjecture of Khot implies that the level-2 algorithm gives optimal worst-case approximation ratio for a wide class of problems~\cite{SKhot_2002,Raghavendra_2008_optimalitySDP_UG,SKhot_2010}. Sum-of-squares algorithms have also seen many success stories for average-case inference problems such as planted sparse vector \cite{Barak_Kelner_Steurer_SOS,sos-fast}, dictionary learning \cite{sos-dictionary}, tensor PCA \cite{Hopkins_Shi_Steurer_fastTensorPCASOS}, tensor decomposition \cite{sos-dictionary,ge-ma-decomp,sos-fast,sos-tensor-decomp-poly}, and tensor completion \cite{barak-moitra,sos-tensor-completion-exact}. One way to argue that an inference problem is hard is by showing that the sum-of-squares hierarchy fails to solve it at a particular level $d$ (or ideally, at every constant value of $d$). Such lower bounds have been shown for many problems such as planted clique~\cite{Baraketal_PlantedClique} and tensor PCA~\cite{Hopkins_Shi_Steurer_fastTensorPCASOS}. There is also recent work that gives evidence for computational hardness by relating the power of sum-of-squares to the low-degree moments of the posterior distribution \cite{hs-few-samples}.

\subsubsection*{Belief propagation, approximate message passing, and the cavity method}

Another important heuristic to predict computational thresholds is based on ideas from statistical physics and is often referred to as the \emph{cavity method} \cite{mezard-parisi-virasoro}. It is based on analyzing an iterative algorithm called \emph{belief propagation} (BP) \cite{pearl}, or its close relative \emph{approximate message passing} (AMP) \cite{amp}. Specifically, BP has a trivial fixed point wherein the algorithm fails to perform inference. If this fixed point is stable (attracting) then we expect inference to be computationally hard. In these notes we will cover this method in detail. For further references, see \cite{mm-book,statmech-survey}.

\subsubsection*{Replica method and the complexity of the posterior}

Another method borrowed from statistical physics is the \emph{replica method} (see e.g.\ \cite{mm-book}). This is a mysterious non-rigorous calculation from statistical physics that can produce many of the same predictions as the cavity method. One way to think about this method is as a way to measure the complexity of the posterior distribution. In particular, we are interesting in whether the posterior distribution resembles one big connected region or whether it fractures into disconnected clusters (indicating computational hardness). We will cover the replica method in Section~\ref{sec:replica} of these notes.

\subsubsection*{Complexity of a random objective function} 

Another method for investigating computational hardness is through the lens of non-convex optimization. Intuitively, we expect that ``easy'' optimization problems have no ``bad'' local minima and so an algorithm such as gradient descent can find the global minimum (or at least a point whose objective value is close to the global optimum). For Bayesian inference problems, maximum likelihood estimation amounts to minimizing a particular random non-convex function. One tool to study critical points of random functions is the \emph{Kac-Rice formula} (see~\cite{AdlerTaylor_book} for an introduction). This has been used to study optimization landscapes in settings such as spin glasses \cite{abac}, tensor decomposition \cite{Ge_Ma_LandscapeTensorDecomposition}, and problems arising in community detection~\cite{Bandeira_SDPlowrank_COLT}. There are also other methods to show that there are no spurious local minima in certain settings, e.g.\ \cite{ge-no-spurious,Boumal_SDPlowrank_NIPS,Venturi_NNnospuriousvalleys}.

\section{Setting and vocabulary}

Throughout, we'll largely focus on Bayesian inference problems. Here we have a signal $\sigma^* \in \RR^n$ viewed through some noisy observation model. We present two examples, and examine them simultaneously through the parallel language of machine learning and statistical physics.

\begin{example}[Rademacher spiked Wigner]\label{ex:radwig}
The signal $\sigma^*$ is drawn uniformly at random from $\{\pm 1\}^n$. We observe the $n \times n$ matrix
$$ Y = \frac{\lambda}{n} \sigma^* (\sigma^*)^\top + \frac{1}{\sqrt{n}} W, $$
where $\lambda$ is a signal-to-noise parameter, and $W$ is a GOE matrix\footnote{Gaussian orthogonal ensemble: symmetric with the upper triangle drawn \iid as $\cN(0,1)$.}. We wish to approximately recover $\sigma^*$ from $Y$, up to a global negation (since $\sigma^*$ and $-\sigma^*$ are indistinguishable).
\end{example}

This problem is motivated by the statistical study of the \emph{spiked Wigner model} from random matrix theory (see e.g.\ \cite{pwbm-contig}). This model has also been studied as a Gaussian variant of community detection \cite{dam} and as a model for synchronization over the group $\ZZ/2$ \cite{sdp-phase}.

\begin{example}[Stochastic block model]\label{ex:cdsbm}
The signal $\sigma^*$ is drawn uniformly at random from $\{\pm 1\}^n$. We observe a graph $G$ with vertex set $[n] = \{1,\ldots,n\}$, with edges drawn independently as follows: for vertices $u,v$, we have $u \sim v$ with probability $a/n$ if $\sigma_u \sigma_v = 1$, and probability $b/n$ if $\sigma_u \sigma_v = -1$. We will restrict ourselves to the case $a > b$. We imagine the entries $\sigma_u^*$  as indicating membership of vertex $u$ in either the $+1$ or $-1$ `community'; thus vertices in the same community are more likely to share an edge. We wish to approximately recover the community structure $\sigma^*$ (up to global negation) from $G$.
\end{example}

This is a popular model for community detection in graphs. See e.g. \cite{abbe-survey,moore-survey} for a survey. Here we consider the sparse regime, but other regimes are also considered in the literature.

There is a key difference between the two models above. The Rademacher spiked Wigner model is \emph{dense} in the sense that we are given an observation for every pair of variables. On the other hand, the stochastic block model is \emph{sparse} (at least in the regime we have chosen) because essentially all the useful information comes from the observed edges, which form a sparse graph. We will see that different tools are needed for dense and sparse problems.

\subsection{Machine learning view.}
We are interested in inferring the signal $\sigma^*$, so it is natural to write down the posterior distribution. For the Rademacher spiked Wigner problem (Example~\ref{ex:radwig}), we can compute the posterior distribution explicitly as follows:
\begin{align*}
    \Pr[\sigma \given Y] \propto \Pr[Y \given \sigma] &\propto \prod_{i < j} \exp\left( -\frac{n}{2} \left(Y_{ij} - \frac{\lambda}{n} \sigma_i \sigma_j \right)^2 \right) \\
    &= \prod_{i < j} \exp\left( -\frac{n}{2} Y_{ij}^2 + \lambda Y_{ij} \sigma_i \sigma_j - \frac{\lambda^2}{2n} \sigma_i^2 \sigma_j^2 \right) \\
    &\propto \prod_{i < j} \exp\left( \lambda Y_{ij} \sigma_i \sigma_j \right).
\end{align*}
(Here $\propto$ hides a normalizing constant which depends on $Y$ but not $\sigma$; it is chosen so that $\sum_{\sigma \in \{\pm 1\}^n} \Pr[\sigma \given Y] = 1$.) The above factorization over edges defines a \defn{graphical model}: a probability distribution factoring in the form $\Pr[\sigma] = \prod_{S \subset [n]} \psi_S(\sigma_S)$ into \defn{potentials} $\psi_S$ that each only depend on a small (constant-size) subset $S$ of the entries of $\sigma$. (For instance, in our example above, $S$ ranges over all subsets of size 2.)

\subsection{Statistical physics view.} The observation $Y$ defines a \defn{Hamiltonian}, or energy function, $H(\sigma) = \sum_{i < j} Y_{ij} \sigma_i \sigma_j$, consisting of two-spin interactions; we refer to each entry of $\sigma$ as a \defn{spin}, and to $\sigma$ as a \defn{state}. A Hamiltonian together with a parameter $T = \frac{1}{\beta}$, called the \defn{temperature}, defines a \defn{Gibbs distribution} (or \defn{Boltzmann distribution}):
$$ \Pr[\sigma] \propto e^{-\beta H(\sigma)}. $$
Thus low-energy states are more likely than high-energy states; moreover, at low temperature (large $\beta$), the distribution becomes more concentrated on lower energy states, becoming supported entirely on the minimum energy states (\defn{ground states}) in the limit as $\beta \to \infty$. On the other hand, in the high-temperature limit ($\beta \to 0$), the Gibbs distribution becomes uniform.

Connecting the ML and physics languages, we observe that the posterior distribution on $\sigma$ is precisely the above Gibbs distribution, at the particular inverse-temperature $\beta = \lambda$. (This is often referred to as lying on the \defn{Nishimori line} \cite{nish-1,nish-2,nish-book}, or being at Bayes-optimal temperature.)

\subsection{Optimization and statistical physics.}
A common optimization viewpoint on inference is \defn{maximum likelihood estimation}, or the maximization task of finding the state $\sigma$ that maximizes the posterior likelihood. This optimization problem is frequently computationally hard, but convex relaxations or surrogates may be studied. To rephraze this optimization task in physical terms, we wish to minimize the energy $H(\sigma)$ over states $\sigma$, or equivalently sample from (or otherwise describe) the low-temperature Gibbs distribution in the limit $\beta \to \infty$.

This viewpoint is limited, in that the MLE frequently lacks any \emph{a priori} guarantee of optimality. On the other hand, the Gibbs distribution at the true temperature $\beta = \lambda$ enjoys optimality guarantees at a high level of generality:
\begin{claim} Suppose we are given some observation $Y$ leading to a posterior distribution on $\sigma$. For any estimate $\hat\sigma = \hat\sigma(Y)$, define the (expected) mean squared error (MSE) $\EE \|\hat\sigma - \sigma\|_2^2$. The estimator that minimizes the expected MSE is given by $\hat\sigma = \EE[\sigma \given Y ]$, the posterior expectation (and thus the expectation under the Gibbs distribution at Bayes-optimal temperature).
\end{claim}

\begin{remark}
In the case of the Rademacher spiked Wigner model, there is a caveat here: since $\sigma^*$ and $-\sigma^*$ are indistinguishable, the posterior expectation is zero. Our objective is not to minimize the MSE but to minimize the error between $\sigma$ and either $\hat\sigma$ or $-\hat\sigma$ (whichever is better).
\end{remark}

Thus the optimization approach of maximum likelihood estimation aims for too low a temperature.
aggregate likelihood. Intuitively, MLE searches for the single state with highest individual likelihood, whereas the optimal Bayesian approach looks for a large cluster of closely-related states with a high aggregate likelihood.

Fortunately, the true Gibbs distribution has an optimization property of its own:
\begin{claim} The Gibbs distribution with Hamiltonian $H$ and temperature $T > 0$ is the unique distribution minimizing the \defn{(Helmholtz) free energy}
$$ F = \mathbb{E}H - T S, $$
where $S$ denotes the Shannon entropy $S = - \EE_\sigma \log \Pr(\sigma)$.
\end{claim}
\begin{proof}
Entropy is concave with infinite derivative at the edge of the probability simplex, and the expected Hamiltonian is linear in the distribution, so the free energy is convex and minimized in the interior of the simplex. We find the unique local (hence global) minimum with a Lagrange multiplier:
\begin{align*}
    \mathrm{const} \cdot \one &= \nabla F = \nabla \sum_\sigma p(\sigma) (H(\sigma) + T \log p(\sigma)) \\
    \mathrm{const} &= H(\sigma) + T \log p(\sigma) \\
    \mathrm{const} \cdot e^{-H(\sigma)/T} &= p(\sigma),
\end{align*}
which we recognize as the Gibbs distribution.
\end{proof}
This optimization approach is willing to trade off some energy for an increase in entropy, and can thus detect large clusters of states with a high aggregate likelihood, even when no individual state has the highest possible likelihood. Moreover, the free energy is convex, but it is a function of an arbitrary probability distribution on the state space, which is typically an exponentially large object.

We are thus led to ask the question: is there any way to reduce the problem of free energy minimization to a tractable, polynomial-size problem? Can we get a theoretical or algorithmic handle on this problem?

\section{The cavity method and belief propagation}
\label{sec:cavity}

\subsection{BP as an algorithm for inference}

\defn{Belief propagation} (BP) is a general algorithm for inference in graphical models, generally credited to Pearl \cite{pearl} (see e.g.\ \cite{mm-book} for a reference). As we've seen above, the study of graphical models is essentially the statistical physics of Hamiltonians consisting of interactions that each only depend on a few spins. Quite often, we care about the \defn{average case} study of random graphical models that describe a posterior distribution given some noisy observation of a signal, such as in the Rademacher spiked Wigner example discussed above. Much of statistical physics is concerned with disorder and random systems, and indeed the concept of belief propagation appeared in physics as the \defn{cavity method}---not only as an algorithm but as a theoretical means to make predictions about systems such as spin glasses \cite{mezard-parisi-virasoro}.

To simplify the setting and notation, let us consider sparse graphical models with only pairwise interactions:
$$ \Pr[\sigma] \propto \prod_{u \sim v} \psi_{uv}(\sigma_u,\sigma_v), $$
where each vertex $v$ has only relatively few ``neighbors'' $u$ (denoted $u \sim v$).

Belief propagation is an iterative algorithm. We think of each spin $\sigma_u$ as a vertex and each pair of neighbors as an edge. Each vertex tracks a ``belief'' about its own spin (more formally, an estimated posterior marginal). These beliefs are often initialized to something like a prior distribution, or just random noise, and then iteratively refined to become more consistent with the graphical model. This refinement happens as follows: each vertex $u$ transmits its belief to each neighbor, and then each vertex updates its belief based on the incoming beliefs of its neighbors. If we let $m_{u \to v}$ denote the previous beliefs sent from neighbors $u$ to a vertex $v$, we can formulate a new belief for $v$ in a Bayesian way, assuming that the incoming influences of the neighbor vertices are independent (more on this assumption below):
$$ m_v(\sigma_v) \propto \prod_{u \sim v} \sum_{\sigma_u} \psi_{uv}(\sigma_u,\sigma_v) m_{u \to v}(\sigma_u) $$
Each message $m_{u \to v}$ is a probability distribution (over the possible values for $\sigma_u$), with the proportionality constant being determined by probabilities summing to $1$ over all values of $\sigma_u$.

This is almost a full description of belief propagation, except for one detail. If the belief from vertex $v$ at time $t-2$ influences the belief of neighbor $u$ at time $t-1$, then neighbor $u$ should not parrot that influence back to neighbor $v$, reinforcing its belief at time $t$ without any new evidence. Thus we ensure that the propagation of messages does not immediately backtrack:
\begin{equation} m_{v \to w}^{(t)}(\sigma_v) \propto \prod_{\substack{u \sim v \\ u \neq w}} \sum_{\sigma_u} \psi_{uv}(\sigma_u,\sigma_v) m_{u \to v}^{(t-1)}(\sigma_u). \label{eq:bp} \end{equation}
This formula is the iteration rule for belief propagation.

The most suspicious aspect of the discussion above is the idea that neighbors of a vertex $v$ exert probabilistically independent influences on $v$. If the graphical model is a tree, then the neighbors are independent after conditioning on $v$, and in this setting it is a theorem (see e.g.\ \cite{mm-book}) that belief propagation converges to the exact posterior marginals. On a general graphical model, this independence fails, and belief propagation is heuristic. In many sparse graph models, neighborhoods of most vertices are trees, with most loops being long, so that independence might approximately hold. BP certainly fails in the worst case; outside of special cases such as trees it is certainly only suitable in an average-case setting. However, on many families of random graphical models, belief propagation is a remarkably strong approach; it is general, efficient, and often yields a state-of-the-art statistical estimate. It is conjectured in many models that belief propagation achieves asymptotically optimal inference, either among all estimators or among all polynomial-time estimators, but most rigorous results in this direction are yet to be established.

To connect to the previous viewpoint of free energy minimization: belief propagation is intimately connected with the \defn{Bethe free energy}, a heuristic proxy for the free energy which may be described in terms of the messages $m_{u \to v}$ (see \cite{statmech-survey}, Section III.B). It can be shown that the fixed points of BP are precisely the critical points of the Bethe free energy, justifying the view that BP is roughly a minimization procedure for the free energy. Again, rigorously the situation is much worse: the Bethe free energy is non-convex, and BP is not guaranteed to converge, let alone guaranteed to find the global minimum.

\subsection{The cavity method for the stochastic block model}

The ideas of belief propagation above appear as the \defn{cavity method} in statistical physics, owing to the idea that the Bethe free energy is believed to be essentially an accurate model for the true (Helmholtz) free energy on a variety of models of interest. In passing to the Bethe free energy, we can pass from studying a general distribution (an exponentially complicated object) to studying node and edge marginals, which are theoretically much simpler objects and, crucially, can be studied locally on the graph. Local neighborhoods of sparse graphs as in the SBM (stochastic block model) look like trees, and so we are drawn to studying message passing on a tree.

Much as for the Rademacher spiked Wigner model above, we derive a Hamiltonian from the block model posterior:
$$ H(\sigma) = \sum_{i \sim j} \theta_+ \sigma_i \sigma_j + \sum_{i \not\sim j} \theta_- \sigma_i \sigma_j, $$
where $u \sim v$ denotes adjacency in the observed graph, and $\theta_+ > 0 > \theta_-$ are constants depending on $a$ and $b$; $\theta_+$ is of constant order, while $\theta_-$ is of order $1/n$.
In expressing belief propagation, we will make a small notational simplification: instead of passing messages $m$ that are distributions over $\{+,-\}$, it suffices to pass the expectation $m(+) - m(-)$. The reader can verify that rewriting the belief propagation equations in this notation yields
$$ m_{u \to v}^{(t)} = \tanh\left( \sum_{\substack{w \sim u \\ w \neq v}} \atanh(\theta_+ m_{w \to u}^{(t-1)}) + \sum_{\substack{w \not\sim u \\ w \neq v}} \atanh(\theta_- m_{w \to u}^{(t-1)}) \right) $$
where $\tanh$ is the hyperbolic tangent function $\tanh(z) = (e^z - e^{-z})/(e^z + e^{-z})$, and $\atanh$ is its inverse.

The first term inside the $\tanh$ represents strong, constant-order attractions with the few graph neighbors, while the second term represents very weak, low-order repulsions with the multitude of non-neighbors. The value of the second term thus depends very little on any individual spin, but rather on the overall balance of positive and negative spins in the graph, with the tendency to cause the global spin configuration to become balanced. As we are only interested in a local view of message passing, we will assume here that the global configuration is roughly balanced and neglect the second term:
$$ m_{u \to v}^{(t)} \approx \tanh\left( \sum_{\substack{w \sim u \\ w \neq v}} \atanh(\theta_+ m_{w \to u}^{(t-1)}) \right). $$

As this message-passing only involves the graph edges, it now makes sense to study this on a tree-like neighborhood. We now discuss a generative model for (approximate) local neighborhoods under the stochastic block model.
\begin{model}[Galton--Watson tree] Begin with a root vertex, with spin $+$ or $-$ chosen uniformly. Recursively, each vertex gives birth to a Poisson number of child nodes: $\Pois((1-\eps)k)$ vertices of the same spin and $\Pois(\eps k)$ vertices of opposite spin, up to a total tree depth of $d$.
\end{model}
As shown in \cite{mns}, the Galton--Watson tree with $k = (a+b)/2$ and $\eps = b / (a+b)$ is distributionally very close to the radius-$d$ neighborhood of a vertex in the SBM with its true spins, so long as $d = o(\log n)$. Thus we will study belief propagation on a random Galton--Watson tree.

Let us consider only the BP messages passing toward the root of the tree. The upward message from any vertex $v$ is computed as:
\begin{equation} m_v = \tanh\left( \sum_u \atanh((1-2\eps) m_u) \right) \label{eq:treebp} \end{equation}
where $u$ ranges over the children of $v$. We now imagine that the child messages $m_u$ are independently drawn from some distribution $D_+^{(t-1)}$ for children with spin $+$, and (leveraging symmetry) from the distribution $D_-^{(t-1)} = -D_+^{(t-1)}$ for children with spin $-$; this distribution represents the randomness of our BP calculation below each child, over the random subtree hanging off each one. Then, from equation \eqref{eq:treebp}, together with the fact that there are $\Pois((1-\eps)k)$ same-spin children and $\Pois(\eps k)$ opposite-spin children, the distribution $D_{\pm}^{(t)}$ of the parent message $m$ is determined! Thus we obtain a distributional recurrence for $D_{+}^{(t)}$.

The calculation above is independent of $n$, and the radius of validity of the tree approximation grows with $n$, so we are interested in the behavior of the recurrence above as $t \to \infty$, i.e.\ fixed points of the distributional recurrence above and their stability. 

Typically one initializes BP with small random messages, a perturbation of the trivial all-$0$ fixed point that represents our prior. For small messages, we can linearize $\tanh$ and $\atanh$, and write $m_v \approx (1-2\eps) \sum_u m_u$. Then if the child distribution $D_+^{(t-1)}$ has mean $\mu$ and variance $\sigma^2$, it is easily computed that the parent distribution $D_+^{(t)}$ has mean $k(1-2\eps)^2 \mu$ and variance $k(1-2\eps)^2 \sigma^2$. Thus if $k (1-2\eps)^2 < 1$, then perturbations of the all-$0$ fixed point decay, or in other words, this fixed point is stable, and BP is totally uninformative on this typical initialization. If $k(1-2\eps)^2 > 1$, then small perturbations do become magnified under BP dynamics, and one imagines that BP might find a more informative fixed point (though this remains an open question!).

This transition is known as the \defn{Kesten--Stigum threshold} \cite{kesten-stigum}, and calculations of this form are loosely conjectured to describe the \defn{computational threshold} beyond which no efficient algorithm can perform inference, for sparse models such as the SBM. In this particular form of the SBM, this idea has been rigorously vindicated using techniques slightly different from BP: inference is known to be statistically impossible when $k(1-2\eps)^2 < 1$ (meaning that any estimator has zero correlation with the truth as $n \to \infty$) \cite{mns}, and efficiently possible when $k(1-2\eps)^2 > 1$ (meaning that asymptotically nonzero correlation is possible) \cite{massoulie,mns2}.

One might also endeavor to study the other fixed points of BP, not just the trivial fixed points. This is a difficult undertaking in most situations, as the BP recurrence lacks convexity properties, but it is expected to give an understanding of the \defn{statistical threshold} of the problem, i.e.\ the limit below which even inefficient inference techniques fail. This has been rigorously proven for some variants of the stochastic block model \cite{coja-cavity}. Intuitively, exploring the BP landscape by brute force for the best (in terms of Bethe free energy) BP fixed point is a statistically optimal inference technique. For more general stochastic block models with 4 or more communities, there exists a gap between the statistical threshold and the analogous Kesten--Stigum bound \cite{decelle,abbe-detection,bmnn}.

\section{Approximate message passing}
\label{sec:amp}

\subsection{AMP as a simplification of BP}

Our cavity analysis of the block model above was well-adapted to sparse models, in which the analysis localizes onto a tree of constant average degree. But many models, such as the Rademacher spiked Wigner model, are dense and their analysis cannot be local. Thankfully, many of these models are amenable to analysis for different reasons: as each vertex is acted on by a large number of individually weak influences, the quantities of interest in belief propagation are subject to central limit theorems and concentration of measure. In this section we will demonstrate this on the Rademacher spiked Wigner example.

Recall the Hamiltonian $H = \sum_{i < j} Y_{ij} \sigma_i \sigma_j$ and inverse temperature $\lambda$. As in the SBM discussion above, we can summarize BP messages by the expectation $m(+) - m(-)$. Then BP for this model reads as
$$ m_{u \to v}^{(t)} = \tanh\left( \sum_{w \neq v} \atanh(\lambda Y_{wu} m_{w \to u}^{(t-1)}) \right). $$
We next exploit the weakness of individual interactions. Note that the values $m_{w \to u}^{(t-1)}$ lie in $[-1,1]$, while $Y_{wu}$ is of order $n^{-1/2}$ in probability. Taylor-expanding $\atanh$, we simplify:
$$ m_{u \to v}^{(t)} = \tanh\left( \left(\sum_{w \neq v} \lambda Y_{wu} m_{w \to u}^{(t-1)} \right) + O(n^{-1/2}) \right) \quad \text{w.h.p.} $$

We next simplify the non-backtracking nature of BP. Na\"ively, one might expect that we can simply drop the condition $w \neq v$ from the sum above, as the contribution from vertex $v$ in the above sum should be only of size $n^{-1/2}$. As our formula for $m_{u \to v}^{(t)}$ would then no longer depend on $v$, we could write down messages indexed by a single vertex:
$$ m_u^{(t)} = \tanh\left( \sum_w \lambda Y_{wu} m_w^{(t-1)} \right), $$
or in vector notation,
\begin{equation} m^{(t)} = \tanh( \lambda Y m^{(t-1)} ), \label{eq:amp-no-onsager} \end{equation}
where $\tanh$ applies entrywise. This resembles the ``power iteration'' iterative algorithm to compute the leading eigenvector of $Y$:
$$ m^{(t)} = Y m^{(t-1)}, $$
but with $\tanh(\lambda \;\bullet\; )$ providing some form of soft projection onto the interval $[-1,1]$, exploiting the entrywise $\pm 1$ structure.

Unfortunately, the non-backtracking simplification above is flawed, and equation \eqref{eq:amp-no-onsager} does not accurately summarize BP or provide as strong an estimator. The problem is that the terms we have neglected add up constructively over two iterations. Specifically: consider that vertex $v$ exerts an influence $\lambda Y_{vu} m_v^{(t-2)}$ on each neighbor $u$; this small perturbation translates directly to a perturbation of $m_u^{(t-1)}$ (scaled by a derivative of $\tanh$). At the next iteration, vertex $u$ influences $m_v^{(t)}$ according to $\lambda Y_{vu} m_u^{(t-1)}$; the total contribution from backtracking here is thus $\lambda^2 Y_{vu}^2 m_v^{(t-2)}$, scaled through some derivatives of $\tanh$. This influence is a random, positive, order $1/n$ multiple of $m_v^{(t-2)}$. Summing over all neighbors $u$, we realize that the aggregate contribution of backtracking over two steps is in fact of order $1$.

Thankfully, this contribution is also a sum of small random variables, and exhibits concentration of measure. The solution is thus to subtract off this aggregate backtracking term in expectation, adding a correction called the \defn{Onsager reaction term}:
\begin{equation} m^{(t)} = \tanh\left( Y m^{(t-1)} - \lambda^2 (1-\|m\|_2^2/n) m^{(t-2)} \right). \label{eq:amp}\end{equation} 

This iterative algorithm is known as \defn{approximate message passing} (AMP). The simplifications above to BP first appeared in the work of Thouless, Anderson, and Palmer \cite{tap}, who used it to obtain a theoretical handle on spin glasses at high temperature. The first AMP algorithm \cite{amp} appeared in the context of compressed sensing. The AMP algorithm \eqref{eq:amp} for this problem can be found in \cite{dam}, and AMP has been applied to many other problems such as rank-one matrix estimation \cite{FR-iterative}, sparse PCA \cite{amp-sparse-pca}, non-negative PCA \cite{amp-nonneg-pca}, planted clique \cite{amp-clique}, and synchronization over groups \cite{pwbm-amp} (just to name a few).

\subsection{AMP state evolution}

In contrast to belief propagation, approximate message passing (AMP) algorithms tend to be amenable to exact analysis in the limit $n \to \infty$. Here we introduce \emph{state evolution}, a simple heuristic argument for the analysis of AMP that has been proven correct in many settings. The idea of state evolution was first introduced by \cite{amp}, based on ideas from \cite{bolthausen}; it was later proved correct in various settings \cite{bm,jm}.

We will focus again on the Rademacher spiked Wigner model: we observe
$$Y = \frac{\lambda}{n} x x^\top + \frac{1}{\sqrt n} W$$
where $x \in \{\pm 1\}^n$ is the true signal (drawn uniformly at random) and the $n \times n$ noise matrix $W$ is symmetric with the upper triangle drawn \iid as $\mathcal{N}(0,1)$. In this setting, the AMP algorithm and its analysis are due to \cite{dam}.

We have seen above that the AMP algorithm for this problem takes the form
$$v^{t+1} = Y f(v^t) + [Onsager]$$
where $f(v)$ denotes entrywise application of the function $f(v) = \tanh(\lambda v)$. (Here we abuse notation and let $f$ refer to both the scalar function and its entrywise application to a vector.) The superscript $t$ indexes timesteps of the algorithm (and is not to be confused with an exponent). The details of the Onsager term, discussed previously, will not be important here.

The state evolution heuristic proceeds as follows. Postulate that at timestep $t$, AMP's iterate $v^t$ is distributed as
\begin{equation}
\label{eq:se}
v^t = \mu_t x + \sigma_t g \quad \text{where } g \sim \mathcal{N}(0,I).
\end{equation}
This breaks down $v^t$ into a signal term (recall $x$ is the true signal) and a noise term, whose sizes are determined by parameters $\mu_t \in \RR$ and $\sigma_t \in \RR_{\ge 0}$. The idea of state evolution is to write down a recurrence for how the parameters $\mu_t$ and $\sigma_t$ evolve from one timestep to the next. In performing this calculation we will make two simplifying assumptions that will be justified later: (1) we drop the Onsager term, and (2) we assume the noise $W$ is independent at each timestep (i.e.\ there is no correlation between $W$ and the noise $g$ in the current iterate). Under these assumptions we have
\begin{align*}
v^{t+1} &= Yf(v^t) = \left(\frac{\lambda}{n} x x^\top + \frac{1}{\sqrt n} W \right)f(v^t) \\
&= \frac{\lambda}{n} \langle x, f(v^t) \rangle\, x + \frac{1}{\sqrt n} W f(v^t)
\end{align*}
which takes the form of (\ref{eq:se}) with a signal term and a noise term. We therefore have
\begin{align*}
\mu_{t+1} &= \frac{\lambda}{n} \langle x,f(v^t) \rangle = \frac{\lambda}{n} \langle x,f(\mu_t x + \sigma_t g) \rangle \\
& \approx \lambda \Ex_{X,G}[X f(\mu_t X + \sigma_t G)] \quad \text{with scalars } X \sim \mathrm{Unif}\{\pm 1\}, G \sim \mathcal{N}(0,1) \\
&= \lambda \Ex_G [f(\mu_t + \sigma_t G)] \quad\text{since $f(-v) = -f(v)$.}
\end{align*}
For the noise term, think of $f(v^t)$ as fixed and consider the randomness over $W$. Each entry of the noise term $\frac{1}{\sqrt n} W f(v^t)$ has mean zero and variance
\begin{align*}
(\sigma^{t+1})^2 &= \sum_i \frac{1}{n} f(v_i^t)^2 = \sum_i \frac{1}{n} f(\mu_t x_i + \sigma_t g_i)^2 \\
&\approx \lambda \Ex_{X,G}[f(\mu_t X + \sigma_t G)^2] \quad \text{with scalars $X,G$ as above} \\
&= \Ex_G [f(\mu_t + \sigma_t G)^2] \quad \text{again by symmetry of $f$.}
\end{align*}
We now have ``state evolution'' equations for $\mu_{t+1}$ and $\sigma_{t+1}$ in terms of $\mu_t$ and $\sigma_t$. Since we could arbitrarily scale our iterates $v^t$ without adding or losing information, we really only care about the parameter $\gamma \defeq (\mu/\sigma)^2$. It is possible (see \cite{dam}) to reduce the state evolution recurrence to this single parameter:
\begin{equation}
\label{eq:gam}
\gamma_{t+1} = \lambda^2 \Ex_{G \sim \cN(0,1)} \tanh(\gamma_t + \sqrt{\gamma_t} \,G)
\end{equation}
(where we have substituted the actual expression for $f$).

We can analyze AMP as follows. Choose a small positive initial value $\gamma_0$ and iterate (\ref{eq:gam}) until we reach a fixed point $\gamma_\infty$. We then expect the output $v^\infty$ of AMP to behave like
\begin{equation}
\label{eq:se-output}
v^\infty = \mu_\infty x + \sigma_\infty g
\end{equation}
where $g \sim \cN(0,I)$, $\mu_\infty = \gamma_\infty/\lambda$, and $\sigma_\infty^2 = \gamma_\infty/\lambda^2$. For the Rademacher spiked Wigner model, this has in fact been proven to be correct in the limit $n \to \infty$ \cite{bm,jm}. Namely, when we run AMP (with the Onsager term and without fresh noise $W$ at each timestep), the output behaves like (\ref{eq:se-output}) in a particular formal sense.

State evolution reveals a phase transition at $\lambda = 1$: when $\lambda \le 1$ we have $\gamma_\infty = 0$ (so AMP has zero correlation with the truth as $n \to \infty$) and when $\lambda > 1$ we have $\gamma_\infty > 0$ (so AMP achieves nontrivial correlation with the truth). Furthermore, from (\ref{eq:se-output}) we can deduce the value of any performance metric (e.g.\ mean squared error) at any signal-to-noise ratio $\lambda$. It has in fact been shown (for Rademacher spiked Wigner) that the mean squared error achieved by AMP is information-theoretically optimal \cite{dam}.

It is perhaps surprising that state evolution is correct, given the seemingly-questionable assumptions we made in deriving it. This can be understood as follows. Recall that we eliminated the Onsager term and assumed independent noise $W$ at each timestep. Also recall that the Onsager term is a correction that makes the update step non-backtracking: a message sent across an edge at one iteration does not affect the message sent back across the edge (in the opposite direction) at the next iteration. It turns out that to leading order, using fresh noise at each timestep is equivalent to using a non-backtracking update step. This is because the largest effect of fresh noise is to terms where a particular noise entry $W_{ij}$ is used twice in a row, i.e.\ backtracking steps. So the two assumptions we made actually cancel each other out! Note that both of the two assumptions are crucial in making the state evolution analysis tractable, so it is quite spectacular that we are able to make both of these assumptions for free (and still get the correct answer)!

One caveat in the rigorous analysis of AMP is that it assumes an initialization that has some nonzero correlation with the truth \cite{dam}. In other words, we need to assume that we start with some nonzero $\gamma$ because if we start with $\gamma = 0$ we will remain there forever. In practice this is not an issue; a small random initialization suffices.

\subsection{Free energy diagrams}

In this section we will finally see how to predict computational-to-statistical gaps (for dense problems)! Above we have seen how to analyze a particular algorithm: AMP. In various settings it has been shown that AMP is information-theoretically optimal. More generally, it is believed that AMP is optimal among all efficient algorithms (for a wide class of problems). We will now show how to use AMP to predict whether a problem should be easy, (computationally) hard or (statistically) impossible. The ideas here originate from \cite{lkz-mmse,lkz-sparse}.

Recall that the state of AMP is described by a parameter $\gamma$, where larger $\gamma$ indicates better correlation with the truth and $\gamma = 0$ means that AMP achieves zero correlation with the truth. Also recall that the \emph{Bethe free energy} is the quantity that belief propagation (or AMP) is locally trying to minimize. It is possible to analytically write down the function $f(\gamma)$ which gives the (Bethe) free energy of the AMP state corresponding to $\gamma$; in the next section, we will see one way to compute $f(\gamma)$. AMP can be seen as starting near $\gamma = 0$ and naively moving in the direction of lowest free energy until it reaches a local minimum; the $\gamma$ value at this minimum characterizes AMP's output. The information-theoretically optimal estimator is instead described by the global minimum of the free energy (and this has been proven rigorously in various cases \cite{replica-proof,lm}); this corresponds to the inefficient algorithm that uses exhaustive search to find the AMP state which globally minimizes free energy. Figure~\ref{fig:f} illustrates how the free energy landscape $f(\gamma)$ dictates whether the problem is easy, hard, or impossible at a particular $\lambda$ value.

\begin{figure}[!ht]
    \centering
    \begin{subfigure}[t]{0.47\textwidth}
        \centering
        \includegraphics[width=\linewidth]{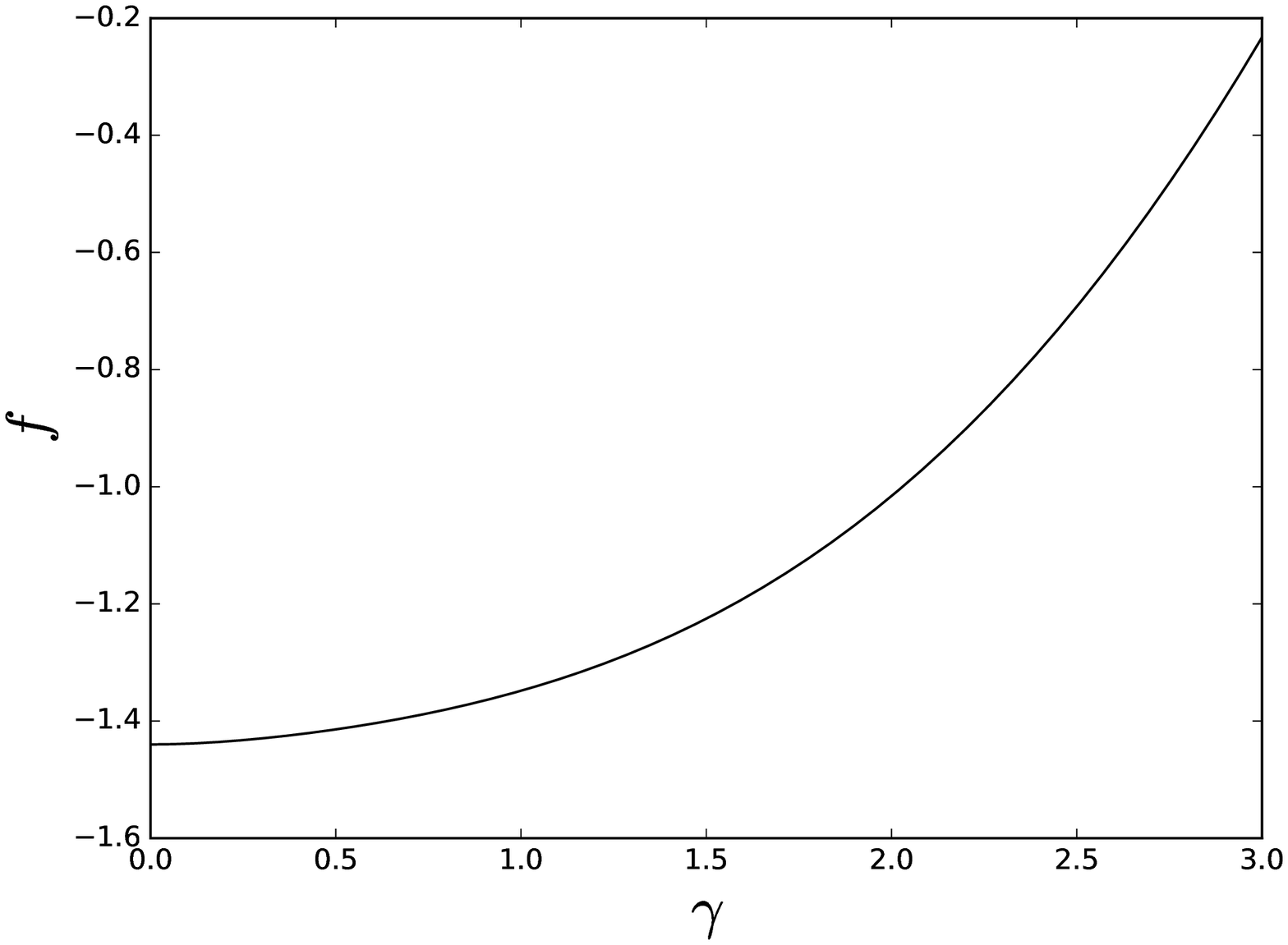}
        \captionof*{figure}{{\bf (a)} impossible}
    \end{subfigure}
    \hfill
    \begin{subfigure}[t]{0.47\textwidth}
        \centering
        \includegraphics[width=\linewidth]{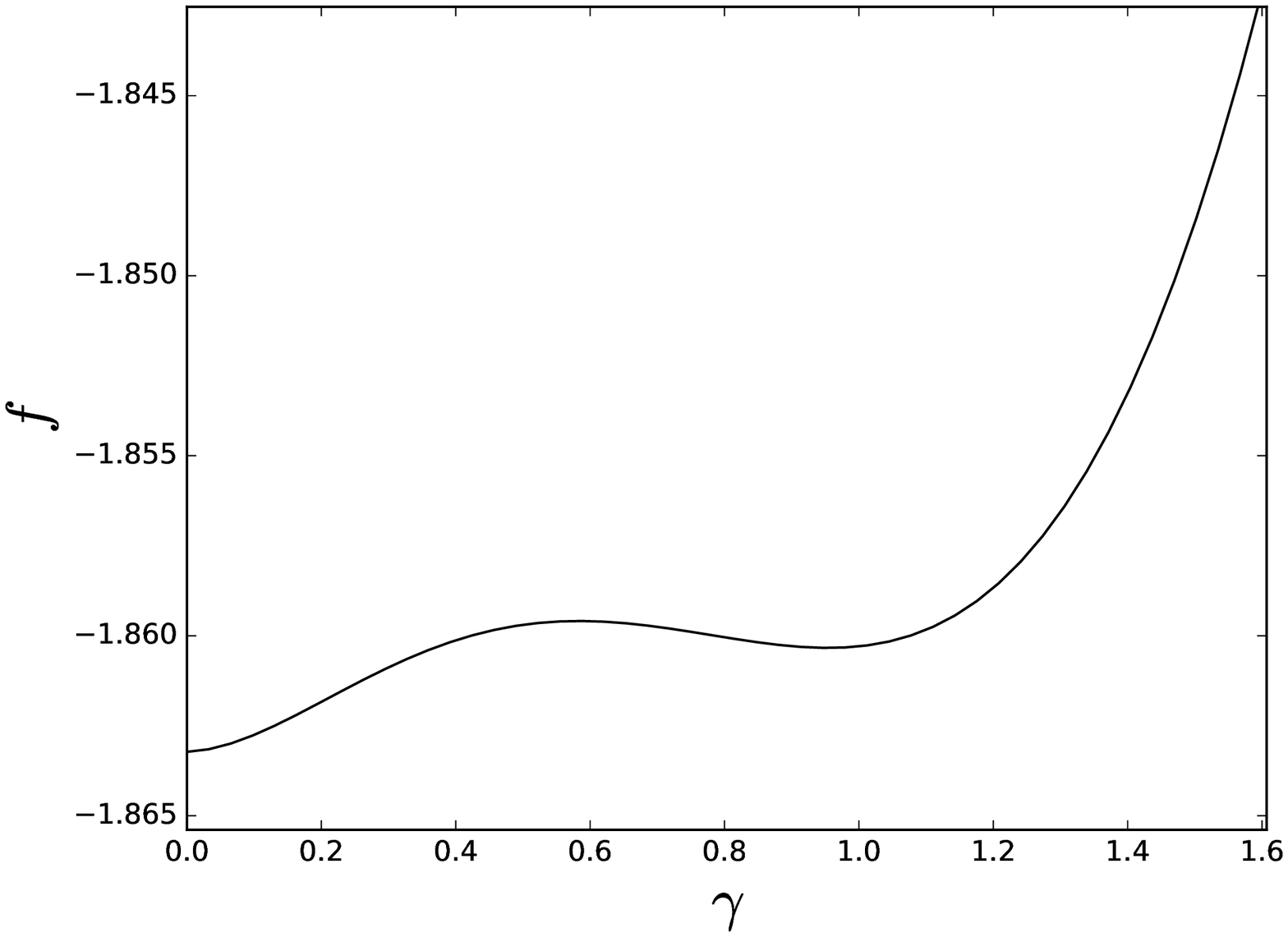}
        \captionof*{figure}{{\bf (b)} impossible}
    \end{subfigure}
    \begin{subfigure}[t]{0.47\textwidth}
        \centering
        \includegraphics[width=\linewidth]{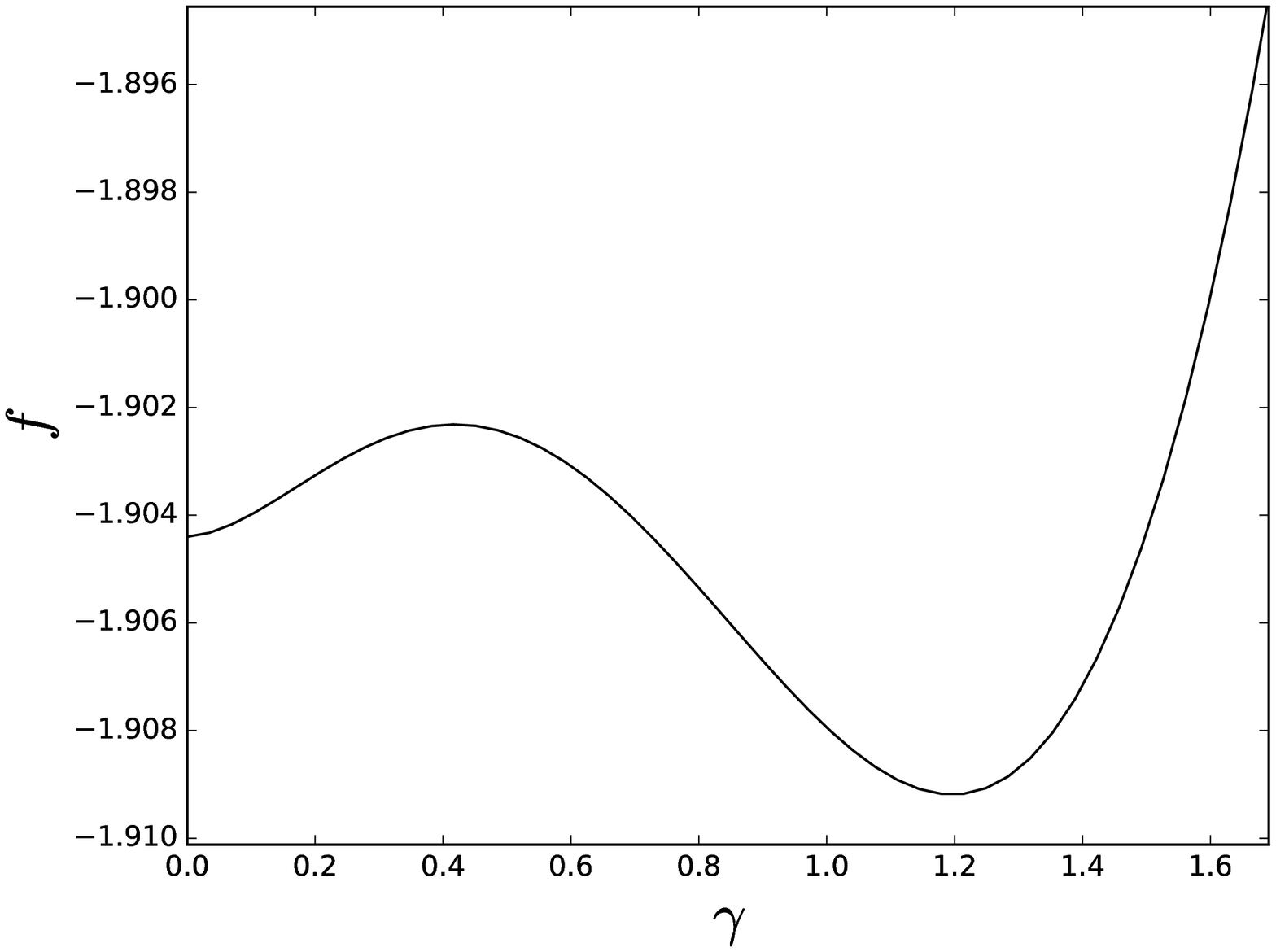}
        \captionof*{figure}{{\bf (c)} hard}
    \end{subfigure}
    \hfill
    \begin{subfigure}[t]{0.47\textwidth}
        \centering
        \includegraphics[width=\linewidth]{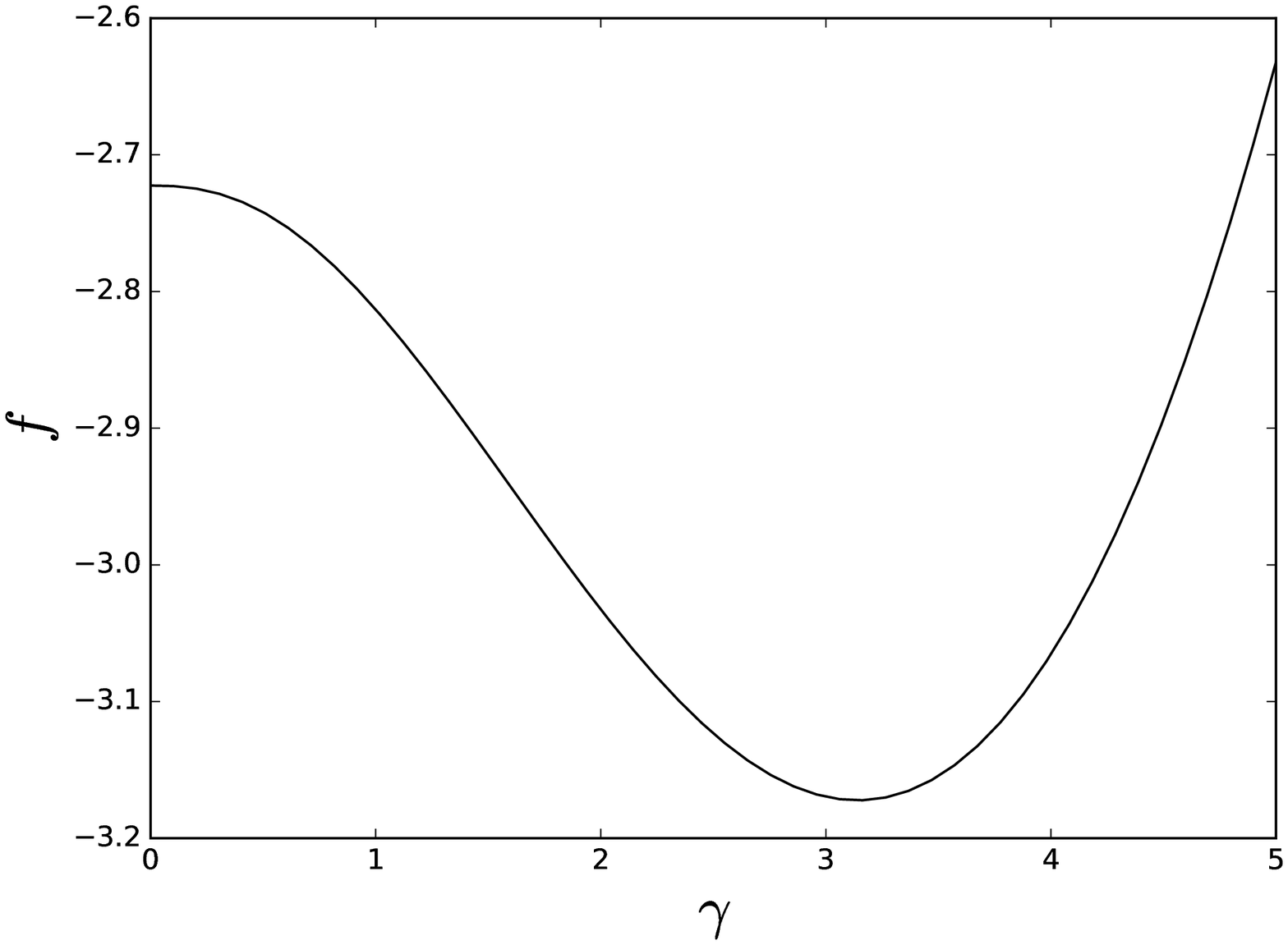}
        \captionof*{figure}{{\bf (d)} easy}
    \end{subfigure}
    \caption{{\bf (a)} The global minimizer is $\gamma = 0$ so no estimator achieves nontrivial recovery. {\bf (b)} A new local minimum in the free energy has appeared, but the global minimum is still at $\gamma = 0$ and so nontrivial recovery remains impossible. {\bf (c)} AMP is stuck at $\gamma = 0$ but the (inefficient) statistically optimal estimator achieves a nontrivial $\gamma$ (the global minimum). AMP is not statistically optimal. {\bf (d)} AMP achieves nontrivial (in fact optimal) recovery. The above image is adapted from \cite{pwbm-amp} (used with permission).}
    \label{fig:f}
\end{figure}

For Rademacher spiked Wigner, we have phase (a) (from Figure~\ref{fig:f}) when $\lambda \le 1$ and phase (d) when $\lambda > 1$, so there is no computational-to-statistical gap. However, for some variants of the problem (for instance if the signal $x$ is sparse, i.e.\ only a small constant fraction of entries are nonzero) then we see phases (a),(b),(c),(d) appear in that order as $\lambda$ increases; in particular, there is a computational-to-statistical gap during the hard phase (c).

Although many parts of this picture have been made rigorous in certain cases, the one piece that we do not have the tools to prove is that no efficient algorithm can succeed during the hard phase (c). This is merely conjectured based on the belief that AMP should be optimal among efficient algorithms.

There are a few different ways to compute the free energy landscape $f(\gamma)$. One method is to use the replica method discussed in the next section. Alternatively, there is a direct formula for Bethe free energy in terms of the BP messages, which can be adapted to AMP (see e.g.\ \cite{lkz-mmse}).

\section{The replica method}
\label{sec:replica}

The replica method is an alternative viewpoint that can derive many of the same results shown in the previous section. We will again use the example of Rademacher spiked Wigner to illustrate it. A general introduction to the replica method can be found in \cite{mm-book}. The calculations of this section are carried out in somewhat higher generality in Appendix~B of \cite{pwb-tensor}.

Recall again the setup: we observe $Y = \frac{\lambda}{n} x x^\top + \frac{1}{\sqrt n} W$ with $x \in \{\pm 1\}^n$ and $W_{ij} = W_{ji} \sim \cN(0,1)$.

The posterior distribution of $x$ given $Y$ is
$$\prob{x \,|\, Y} \propto \prod_{i < j} \exp\left(-\frac{n}{2}\left(\frac{\lambda}{n} x_i x_j - Y_{ij}\right)^2\right)
\propto \exp\left(\lambda \sum_{i < j} Y_{ij} x_i x_j \right)$$
and so we are interested in the Gibbs distribution over $\sigma \in \{\pm 1\}^n$ given by $\prob{\sigma \,|\, Y} \propto \exp(-\beta H(\sigma))$ with energy (Hamiltonian) $H(\sigma) = -\sum_{i < j} Y_{ij} \sigma_i \sigma_j$ and inverse temperature $\beta = \lambda$.

The goal is to compute the free energy density, defined as $f = -\frac{1}{\beta n} \EE \log Z$ where $$Z = \sum_{\sigma \in \{\pm 1\}^n} \exp(-\beta H(\sigma)).$$

\noindent (This can be shown to coincide with the notion of free energy introduced earlier.) The idea of the replica method is to compute the moments $\EE[Z^r]$ of $Z$ for $r \in \mathbb{N}$ and perform the (non-rigorous) analytic continuation
\begin{equation}
\label{eq:replica-trick}
\EE[\log Z] = \lim_{r \to 0} \frac{1}{r} \log \EE[Z^r].
\end{equation}
Note that this is quite bizarre -- we at first assume $r$ is a positive integer, but then take the limit as $r$ tends to zero! This will require writing $\EE[Z^r]$ in an analytic form that is defined for all values of $r$. An informal justification for the correctness of (\ref{eq:replica-trick}) is that when $r$ is close to $0$, $Z^r$ is close to 1 and so we can interchange $\log$ and $\EE$ on the right-hand side.

The moment $\EE[Z^r]$ can be expanded in terms of $r$ `replicas' $\sigma^1,\ldots,\sigma^r$ with $\sigma^a \in \{\pm 1\}^n$:
$$\EE[Z^r] = \sum_{\{\sigma^a\}} \EE\exp\left(\beta \sum_{i < j} Y_{ij} \sum_{a=1}^r \sigma_i^a \sigma_j^a\right).$$
After applying the definition of $Y$ and the Gaussian moment-generating function (to compute expectation over the noise $W$) we arrive at
$$\EE[Z^r] = \sum_{\{\sigma^a\}} \exp\left[n\left(\frac{\lambda^2}{2} \sum_a c_a^2 + \frac{\lambda^2}{4} \sum_{a,b} q_{ab}^2 \right)\right]$$
where $q_{ab} = \frac{1}{n} \sum_i \sigma_i^a \sigma_i^b$ is the correlation between replicas $a$ and $b$, and $c_a = \frac{1}{n} \sum_i \sigma_i^a x_i$ is the correlation between replica $a$ and the truth.

Without loss of generality we can assume (by symmetry) the true spike is $x = \one$ (all-ones). Let $Q$ be the $(r+1)\times (r+1)$ matrix of overlaps ($q_{ab}$ and $c_a$), including $x$ as the zeroth replica. Note that $Q$ is the average of $n$ \iid matrices and so by the theory of large deviations (Cram\'er's Theorem in multiple dimensions), the number of configurations $\{\sigma^a\}$ corresponding to given overlap parameters $q_{ab},c_a$ is asymptotically
\begin{equation}
\label{eq:rad-entropy}
\inf_{\mu,\nu} \exp\left[n\left(-\sum_a \nu_a c_a - \frac{1}{2} \sum_{a \ne b} \mu_{ab} q_{ab} + \log \sum_{\sigma \in \{\pm 1\}^r}\exp\left(\sum_a \nu_a \sigma_a + \frac{1}{2}\sum_{a \ne b} \mu_{ab} \sigma_a \sigma_b\right)\right)\right].
\end{equation}

We now apply the saddle point method: in the large $n$ limit, the expression for $\EE[Z^r]$ should be dominated by a single value of the overlap parameters $q_{ab},c_a$. This yields
$$\frac{1}{n} \log \EE[Z^r] = -G(q_{ab}^*,c_a^*,\mu_{ab}^*,\nu_a^*)$$
where $(q_{ab}^*,c_a^*,\mu_{ab}^*,\nu_a^*)$ is a critical point of
\begin{align*}
G(q_{ab},c_a,\mu_{ab},\nu_a) = &-\frac{\lambda^2}{2} \sum_a c_a^2 - \frac{\lambda^2}{4} \sum_{a,b}  q_{ab}^2 + \sum_a \nu_a c_a + \frac{1}{2} \sum_{a \ne b} \mu_{ab} q_{ab} \\
&- \log \sum_{\sigma \in \{\pm 1\}^r}\exp\left(\sum_a \nu_a \sigma_a + \frac{1}{2}\sum_{a \ne b} \mu_{ab} \sigma_a \sigma_b\right).
\end{align*}

We next assume that the dominant saddle point takes a particular form: the so-called \emph{replica symmetric ansatz}. The replica symmetric ansatz is given by $q_{aa} = 1$, $c_a = c$, $\nu_a = \nu$, and for $a \ne b$, $q_{ab} = q$ and $\mu_{ab} = \mu$ for constants $q,c,\mu,\nu$. This yields
\begin{equation}
\label{eq:G-lim}
\lim_{r \to 0} \frac{1}{r} G(q,c,\mu,\nu) = -\frac{\lambda^2}{2} c^2 - \frac{\lambda^2}{4} + \frac{\lambda^2}{4} q^2 + \nu c - \frac{1}{2} \mu(q-1) - \Ex_{z \sim \cN(0,1)} \log(2 \cosh(\nu + \sqrt{\mu} z))
\end{equation}
where the last term is handled as follows:
\begin{align*}
& \lim_{r \to 0} \frac{1}{r} \log \sum_{\sigma \in \{\pm 1\}^r}\exp\left(\sum_a \nu_a \sigma_a + \frac{1}{2}\sum_{a \ne b} \mu_{ab} \sigma_a \sigma_b\right)\\
=& \lim_{r \to 0} \frac{1}{r} \log \sum_{\sigma \in \{\pm 1\}^r}\exp\left(\nu \sum_a \sigma_a + \frac{\mu}{2}\sum_{a \ne b} \sigma_a \sigma_b\right)\\
=& \lim_{r \to 0} \frac{1}{r} \log \sum_{\sigma \in \{\pm 1\}^r}\exp\left(\nu \sum_a \sigma_a + \frac{\mu}{2} \sum_{a, b} \sigma_a \sigma_b - \frac{r \mu}{2} \right)\\
=& \lim_{r \to 0} \frac{1}{r} \log \sum_{\sigma \in \{\pm 1\}^r}\exp(-r \mu/2)\exp\left(\nu \sum_a \sigma_a + \frac{\mu}{2} \left(\sum_a \sigma_a\right)^2 \right)\\
=& -\frac{\mu}{2} + \lim_{r \to 0} \frac{1}{r} \log \sum_{\sigma \in \{\pm 1\}^r}\exp\left(\nu \sum_a \sigma_a + \frac{\mu}{2} \left(\sum_a \sigma_a\right)^2 \right)\\
\stackrel{(a)}{=}& -\frac{\mu}{2} + \lim_{r \to 0} \frac{1}{r} \log \sum_{\sigma \in \{\pm 1\}^r} \Ex_{z \sim \cN(0,1)} \exp\left(\nu \sum_a \sigma_a + \sqrt{\mu}z \sum_a \sigma_a \right)\\
=& -\frac{\mu}{2} + \lim_{r \to 0} \frac{1}{r} \log \Ex_{z \sim \cN(0,1)} \sum_{\sigma \in \{\pm 1\}^r} \exp\left((\nu + \sqrt{\mu}z) \sum_a \sigma_a \right)\\
=& -\frac{\mu}{2} + \lim_{r \to 0} \frac{1}{r} \log \Ex_{z \sim \cN(0,1)} \left[\exp(\nu + \sqrt{\mu}z) + \exp(-(\nu + \sqrt{\mu}z))\right]^r\\
=& -\frac{\mu}{2} + \lim_{r \to 0} \frac{1}{r} \log \Ex_{z \sim \cN(0,1)} (2 \cosh(\nu + \sqrt{\mu}z))^r\\
\stackrel{(b)}{=}& -\frac{\mu}{2} + \Ex_{z \sim \cN(0,1)} \log(2 \cosh(\nu + \sqrt{\mu}z))
\end{align*}
where (a) uses the Gaussian moment-generating function and (b) uses the replica trick (\ref{eq:replica-trick}).

We next find the critical points by setting the derivatives of (\ref{eq:G-lim}) (with respect to all four variables) to zero, which yields
$$\nu = \lambda^2 c, \qquad \mu = \lambda^2 q, \qquad c = \EE_z \tanh(\nu + \sqrt{\mu}z), \qquad q = \EE_z \tanh^2(\nu + \sqrt{\mu}z).$$

Recall that the replicas are drawn from the posterior distribution $\prob{x \,|\, Y}$ and so the truth $x$ behaves as if it is a replica; therefore we should have $c = q$. Using the identity $\EE_z \tanh(\gamma + \sqrt{\gamma}z) = \EE_z \tanh^2(\gamma + \sqrt{\gamma}z)$ (see e.g.\ \cite{dam}), we obtain the solution $c = q$ and $\nu = \mu$ where $q$ and $\mu$ are solutions to
\begin{equation}
\label{eq:q-mu}
\mu = \lambda^2 q, \qquad q = \Ex_{z \sim \cN(0,1)} \tanh(\mu + \sqrt{\mu}z).
\end{equation}
The solution $q$ to this equation tells us about the structure of the posterior distribution; namely, if we take two independent draws from this distribution, their overlap will concentrate about $q$. (Equivalently, the true signal $x$ and a draw from the posterior distribution will also have overlap that concentrates about $q$.) Note that (\ref{eq:q-mu}) exactly matches the state evolution fixed-point equation (\ref{eq:gam}) with $\mu$ in place of $\gamma$ and $q = \gamma/\lambda^2$.

The free energy density of a solution to (\ref{eq:q-mu}) is given by
$$f = \frac{1}{\beta} \lim_{r \to 0} \frac{1}{r} G(q,c,\mu,\nu) = \frac{1}{\lambda}\left[-\frac{\lambda^2}{4} (q^2 + 1) + \frac{1}{2} \mu (q+1) - \EE_z \log(2 \cosh(\mu + \sqrt{\mu}z))\right].$$
This is how one can derive the free energy curves such as those shown in Figure~\ref{fig:f}. If there are multiple solutions to (\ref{eq:q-mu}), we should take the one with minimum free energy.

Above, we had a Gibbs distribution corresponding to the posterior distribution of a Bayesian inference problem. In this setting, the replica symmetric ansatz is always correct; this is justified by a phenomenon in statistical physics: ``there is no static replica symmetry breaking on the Nishimori line'' (see e.g.\ \cite{statmech-survey,nish-book}).

More generally, one can apply the replica method to a Gibbs distribution that does not correspond to a posterior distribution (e.g.\ if the `temperature' of the Gibbs distribution does not match the signal-to-noise of the observed data). This is important when investigating computational hardness of random non-planted or non-Bayesian problems. In this case, the optimal (lowest free energy) saddle point can take various forms, which are summarized below; the form of the optimizer reveals a lot about the structure of the Gibbs distribution. An important property of a Gibbs distribution is its overlap distribution: the distribution of the overlap between two independent draws from the Gibbs distribution (in the large $n$ limit).

\begin{itemize}
\item RS (replica symmetric): The overlap matrix is $q_{aa} = 1$ and $q_{ab} = q$ for some $q \in [0,1]$. The overlap distribution is supported on a single point mass at value $q$. The Gibbs distribution can be visualized as having one large cluster where any two vectors in this cluster have overlap $q$. This case is ``easy'' in the sense that belief propagation can easily move around within the single cluster and find the true posterior distribution.
\item 1RSB (1-step replica symmetry breaking): The $r \times r$ overlap matrix takes the following form. The $r$ replicas are partitioned into blocks of size $m$. We have $q_{aa} = 1$, $q_{ab} = q_1$ if $a,b$ are in the same block, and $q_{ab} = q_2$ otherwise (for some $q_1,q_2 \in [0,1]$). The overlap distribution is supported on $q_1$ and $q_2$. The Gibbs distribution can be visualized as having a constant number of clusters. Two vectors in the same cluster have overlap $q_1$ whereas two vectors in different clusters have overlap $q_2$. This case is ``hard'' for belief propagation because it gets stuck in one cluster and cannot correctly capture the posterior distribution. The idea of replica symmetry breaking was first proposed in a groundbreaking work of Parisi \cite{parisi-rsb}.
\item 2RSB (2-step replica symmetry breaking): Now we have ``clusters of clusters.'' The overlap matrix has sub-blocks within each block. The overlap distribution is supported on 3 different values (corresponding to ``same sub-block'', ``same block (but different sub-block)'', ``different blocks''). The Gibbs distribution has a constant number of clusters, each with a constant number of sub-clusters. This is again ``hard'' for belief propagation.
\item FSRB (full replica symmetry breaking): We can define kRSB for any $k$ as above (characterized by an overlap distribution supported on $k+1$ values); FRSB is the limit of kRSB as $k \to \infty$. Here the overlap distribution is a continuous distribution.
\item d1RSB (dynamic 1RSB): This phase is similar to RS and (unlike kRSB for $k \ge 1$) can appear in Bayesian inference problems. The overlap matrix is the same as in the RS phase (and so the replica calculation proceeds exactly as in the RS case). However, the Gibbs distribution has exponentially-many small clusters. The overlap distribution is supported on a single point mass because two samples from the Gibbs distribution will be in different clusters with high probability. This phase is ``hard'' for BP (or AMP) because it cannot easily move between clusters. For a Bayesian inference problem, you can tell whether you are in the RS (easy) phase or 1dRSB (hard) phase by looking at the free energy curve; 1dRSB corresponds to the ``hard'' phase (c) in Figure~\ref{fig:f}.
\end{itemize}

\subsection*{Acknowledgements}
The authors would like to thank the engaging audience at the Courant Institute when this material was presented, and their many insightful comments. The authors would also like to thank Soledad Villar and Lenka Zdeborov\'a for feedback on earlier versions of this manuscript.

\small
\bibliographystyle{alpha}
\bibliography{main}

\end{document}